\documentclass[11pt]{article}
\usepackage[a4paper]{geometry}
\usepackage{graphicx}
\usepackage{arydshln}
\usepackage{epsfig}
\usepackage{enumerate}
\usepackage{caption}
\usepackage{subcaption}
\usepackage{wrapfig}
\usepackage{float}
\usepackage[english]{babel}
\usepackage{mathrsfs}
\usepackage{amssymb,amsmath,amsthm}
\usepackage[nottoc]{tocbibind}
\usepackage{xcolor}
\usepackage{sectsty}
\usepackage{titling}
\usepackage{tfrupee}
\usepackage[T1]{fontenc}
\usepackage{titlesec}
\usepackage{blindtext}
\usepackage{parskip}
\usepackage{makeidx}
\usepackage{multirow}
\newtheorem{theorem}{Theorem}
\usepackage{etoolbox}
\usepackage[ruled,vlined]{algorithm2e}
\makeatletter
\providecommand{\subtitle}[1]{
  \apptocmd{\@title}{\par {\large #1 \par}}{}{}
}
\makeatother
\providecommand{\keywords}[1]{\textbf{\textit{Keywords---}} #1}
\title{Improving the Robustness of Federated Learning for Severely Imbalanced Datasets}
\author{Debasrita Chakraborty, Ashish Ghosh}

\date{}

\begin{document}

\maketitle
\abstract{With the ever increasing data deluge and the success of deep neural networks, the research of distributed deep learning has become pronounced. Two common approaches to achieve this distributed learning is synchronous and asynchronous weight update. In this manuscript, we have explored very simplistic synchronous weight update mechanisms. Our motivation is that when multiple models in the nodes are trained on different parts of the data, and the aggregated weight update is performed in the main parameter server, it behaves like an ensemble system or a social learning system. Such a combination of multiple classifiers can handle data complexity very well and supersede a single sophisticated classifier. However, it has been observed that in a parallel setting, this fact is falsified if one uses a synchronous weight update mechanism with simple averaging at the parameter server. It has been seen that with an increasing number of worker nodes, the performance degrades drastically. This effect has been studied in the context of extreme imbalanced classification (e.g. outlier detection). For most distributed DL methods, it is assumed that the worker servers (or local servers) hold independent and identically distributed (i.i.d.) data. However, in practical cases, the conditions of i.i.d. may not be fulfilled. There may also arise global class imbalance situations like that of outlier detection where the local servers receive severely imbalanced data and may not get any samples from the minority class. In that case, the DNNs in the local servers will get completely biased towards the majority class that they receive. This would highly impact the learning at the parameter server (which practically does not see any data). It has been observed that in a parallel setting if one uses the existing federated weight update mechanisms at the parameter server, the performance degrades drastically with the increasing number of worker nodes. This is mainly because, with the increasing number of nodes, there is a high chance that one worker node gets a very small portion of the data, either not enough to train the model without overfitting or having a highly imbalanced class distribution. The chapter, hence, proposes a workaround to this problem by introducing the concept of adaptive cost-sensitive momentum averaging. It is seen that for the proposed system, there was no to minimal degradation in performance while most of the other methods hit their bottom performance before that.}


\section{Introduction}
It is witnessed that AEs have been useful in extracting meaningful feature representations for datasets with imbalanced class distributions \cite{chakraborty2019integration}. Thus, it is intriguing to study how the problem extends to decentralised data settings. AEs have an innate ability to capture the outlier features. Now, the question arises whether the AE pre-initialisation \cite{hinton2006fast} given to a DNN adds any further improvement to the capability of learning from imbalanced datasets. DNNs are inherently biased by the majority class and often fail to give good performance if the classes are severely imbalanced \cite{erfani2016high}. This problem is a big limitation when it comes to achieving high classification performances as compared to balanced classification tasks where DL thrives. DL in distributed setting \cite{gupta2018distributed} is a boon for the current times when we are facing the data deluge. Reports show that human beings have generated more data in just the last decade than have been produced over the last hundred years. This astonishing fact is seen as the foundation for several experiments of concurrent and distributed computation and distributed methodologies in computer science education. Distributed computation is a necessity of this age and it is the pinnacle of this area of science to combine it with sophisticated learning approaches such as DL (neural nets). However, distributed DL in an imbalanced scenario is a big gap that needs to be addressed. If the distributed model uses non-independent and non-identical (or simply non-i.i.d) chunks of data in each local computation server, the imbalance becomes critical and often too overwhelming for a DNN to handle. This is, however, a paradox as multiple experts are being used to model the classification boundary.

This may be viewed as an extension of the concepts of using multiple classifiers to obtain a better classification performance. Such a setting is often regarded as Social Learning System (SLS) \cite{re2012ensemble} or an Ensemble \cite{zhang2012ensemble} or a Multi-Classifier System (MCS) \cite{tamen2017efficient}. This is due to the complementary nature of each of the classifier in the SLS \cite{re2012ensemble}. This diversity is usually achieved either through diversity in data or diversity in models. Deep neural networks under such ensemble setting have been proven to give better performance than their single model counterparts \cite{zhang2012ensemble}. Perceiving that every independent classifier may make complementary errors, we can merge together the decisions from all classifiers to build a composite framework that outperforms any single classifier. Such a combination of multiple classifiers can handle data complexity very well and supersede a single sophisticated classifier. Owing to the training on different inputs, the pool of predictors is ensured to be highly diverse and mutually complementary.

Borrowing the idea of these SLSs, it is expected that a distributed DL model should give better performance than their single model analogy. We can view this particular setting as an ensemble or a social learning system where multiple models in the nodes are trained on different parts of the data, and the aggregated weight update is performed in the main parameter server, it behaves like an SLS. This is also known as federated learning (FL). In the parameter server version of the distributed setting the data samples are spread through a variety of processing nodes and each node has one deep network. There is a parameter server that receives all the parameters (weights and biases) after each communication round, aggregates it and sends the updated aggregated weights to all the worker nodes for the next communication round to begin [Figure \ref{DNU}]. This method of weight update often results in decreased performance as the number of predictors increase. This is mostly due to the fact that the data distribution across the worker nodes are do not have identical distributions and might also have local class imbalance issues. DNNs being just and extension of the traditional neural networks suffer from masking and swamping effects. However, DNNs are state-of-the-art models for most classification tasks. This problem is aggravated if the actual data has an inherent global imbalance (as in case of outlier detection where the class imbalance is severe) and is so big that it needs a distributed processing. As DNNs are inherently unsuitable for the problem of such outlier detection, distributed models fail. However, as identified by many researchers \cite{wang2016auto, protopapadakis2017stacked}, AEs are better suited as outlier detection models. Thus, it would be interesting to see how the AE pre-initialised DNNs perform for outlier detection tasks.

It is even more intriguing to study how an AE pre-initialised distributed architecture of DNNs using a global weight update mechanism at the parameter server would handle the severely imbalanced datasets as with the outlier detection problem. The global data is itself severely imbalanced (outlier class constitute <10\% of the total samples) and thus the local datasets might be critically imbalanced (some may have negligible percentage or completely losing samples from the minority class). Here arises the Hace's theorem of Big Data \cite{wu2013data} when the data is not globally available to all the worker nodes, each of the nodes train independently over any arbitrary portion of the data and make predictions on that.

\begin{figure}
  \centering
      \includegraphics[width=0.7\textwidth]{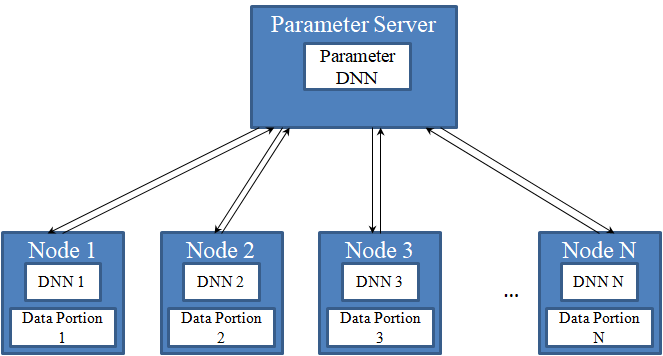}
  \caption{Schematic representation of a federated deep learning setting (DNN represents a deep neural network)}
  \label{DNU}
\end{figure}

In a federated learning \cite{sarkar2020fed, wang2021addressing, zhou2021communication, duan2020self} scenario, the local workers do not communicate among themselves and can only communicate indirectly via the parameter server where they only send the model parameters for aggregation. The parameter server does not contain any data for training. Training is only done in the worker nodes. Each worker node has one deep neural network (DNN) model and a portion of the entire dataset and all the neural networks of each worker node have identical architectures. The idea of FL is that each of the worker nodes do not reveal their local datasets to each other. This property is quite useful in scenarios where data privacy preservation is important among collaborating nodes. The parameter server has as an identical deep neural network only for storing the final updated parameters after each communication round. The basic assumption for such FL implementations are that the data division across the worker nodes are independent and identically distributed (i.i.d) \cite{jacobs1992independent}.

The traditional method, known as federated averaging (FedAvg) \cite{mcmahan2017communication} works simply by averaging all the respective parameters from the worker nodes. This method works well in cases of i.i.d data sample division across worker nodes and is highly scalable. However, for non-i.i.d setting, this method fails as the number of worker nodes increase \cite{wu2021fast}. Moreover, scenarios of class imbalance introduce saturation problem for neural networks \cite{liu2020high}. The non-i.i.d nature of data sample distribution and class imbalance across worker nodes are practical issues and may arise if the local datasets at each worker node come from mostly (or completely) one particular class.

In that case, that particular local model never sees the samples from the minority class. The proposed method handles this intriguing situations when the data is distributed across the local workers such that, the total data is severely imbalanced (in case of under-represented class <10\%) and some of the local workers may not get any samples of the under-represented class. The present manuscript proposes an alternative by introducing a cost-sensitive momentum averaging scheme. The proposed method has been shown to handle these two problems better than the state-of-the-art methods.

The total size of the actual global dataset is constant. In order to get a severely imbalanced global data, we used some datasets with outliers \cite{hawkins2013identification} (one class had <10\% of the total samples). The number of samples each worker node receives is chosen carefully by clustering (details in section \ref{PropDDL}) to avoid human bias. Under such conditions, it has been seen that with increasing number of worker nodes, the performance of the state-of-the-art methods degrades drastically. This degradation is more pronounced when the data is already globally imbalanced. It has been observed that the proposed method is robust in handling the scenario. The manuscript, hence, proposes a workaround by introducing the concept of local cost-sensitive momentum averaging. It is seen that for the proposed system, there was minimal degradation in performance and was robust than the other methods.
\section{Contributions}
In this article, it is shown that AE pre-initialisation strategy works best for outlier detection tasks using a neural network when retraining is done by the focal loss function. Further, this article proposes an adaptive focal loss for learning from severely imbalanced datasets in distributed deep learning setting. In this context, critical class imbalance is explored. We argue that standard methods for federated learning fail to learn from severely imbalanced data and hence face a criticality issue. This is mostly due to either lesser number of samples at the local worker (which makes the local model overfit) or severely imbalanced partitions at the local worker such that most of the workers do not get a single sample from the minority class. We have shown that as the number of workers are increased, the homogeneity \cite{rosenberg2007v} of the local datasets also increase and tend to 1. If all the local datasets contain only a single class' samples, then the homogeneity becomes 1. If the local datasets have perfectly identical representations of the classes as the global dataset then the homogeneity becomes zero. For any severely imbalanced dataset, the homogeneity is seen to increase with the number of workers. So, the assumption, that increasing number of workers result in situations where one or more local datasets have no samples from the minority class (or classes), holds true. So, the major contribution of this work is as follows:
\begin{enumerate}[(i)]
\item study on how AE pre-initialisation helps for imbalanced learning (as in outlier detection tasks),
\item handling outlier detection in federated deep learning, and
\item a novel adaptive focal loss function.
\end{enumerate}

The problem of critical class imbalance in federated deep learning has not been explored or handled in the literature to the best knowledge of the authors. This problem is inherently unique due to its complex nature and resembles a very practical situation (described in section \ref{ddldp}). The basic assumption of i.i.d data distribution across workers is violated in such cases and along with that, the severe class imbalance also affects the performance of the local model at each worker. Related studies have either assumed the data across workers to be i.i.d sampled or having imbalanced class. However, no method has handled the two issues simultaneously.

There may be a criticism of the assumption that the number of samples in the global dataset is constant and that doesn't make the comparison fair on the grounds of unequal number of samples for different number of nodes. Under that light, it needs to be explained that the aim of the experiment is to study a practical situation occurs where small chunks of data are generated by a number of remote devices and data privacy is a concern. The aim was to use the entire training data each time without making any inherent assumption of a local data chunk being more important than the other. This study is different than all the existing studies in this aspect too. The following section describes the related state-of-the art literature around FL.

\section{Related Works}
The manuscript deals with the concepts of FL for classification using pre-trained fully connected deep neural networks. FL in DL is still in its pristine state. The collaborative learning mechanism is perfect for reaping the benefits of DL in distributed data settings. The era of Big Data and the expertise of DL algorithms have forced the researchers to devise novel ways to develop deep distributed learning \cite{gupta2018distributed}. The solutions, however, are problem specific. In cases where the data is too big to be contained in a single machine, but the DL model can be trained in a single machine, the data is  distributed and the model is replicated across different machines. Each of these machines perform local computations on the models and interact with each other directly or through a server to understand a global view of the data without transferring the data as a whole. This is called data-distribution architecture \cite{shallue2019measuring}. In cases where the data is not so big, but the DL model has so many parameters that it cannot be trained in a single machine, the model is  distributed and the data is replicated across different machines. This is called model-distribution architecture \cite{moreno2021heterogeneous}. In cases where both the data as well as the DL model is too big for a single machine, both the data as well as the model are distributed. This is known as hybrid-distribution architecture \cite{oyama2020case}. This manuscript is concerned with data-distribution architecture with the workers and parameter server model shown in Figure \ref{DNU}.

In this context, many researchers have come up with methods that are capable of training deep neural networks where data is distributed across worker nodes, also known as Federated Deep Learning \cite{xu2020privacy}. Each of the worker nodes train a local model on the local dataset that they contain and send the updates (change in weights of the local models) to the parameter server. The parameter server then aggregates the updates or parameters (in this case weights and biases) and sends the aggregated parameters back to the worker nodes for re-initialising the local models after each communication round.

The standard method of aggregation is simple averaging \cite{mcmahan2017communication}, also known as FedAvg. In that case the weights of the parameter server are simple average of the weights of the nodes.
\begin{equation}
w_{param}^t=\dfrac{\sum\limits_{i=1}^N w_i^t}{N},
\end{equation}
where, $w_{param}^t$ is weights of the parameter server model at $t^{th}$ round and $w_i^t$ are the weights of the $i^{th}$ worker's local model at $t^{th}$ round. $N$ represents the number of total models participating in the training. This equation is also often written in terms of the weight change vector ($\Delta w$) too.
\begin{equation}
\Delta w_{param}^t=\dfrac{\sum\limits_{i=1}^N \Delta w_i^t}{N},
\end{equation}
where, $\Delta w_{param}^t$ is the change in weights of the parameter server model at $t^{th}$ round and $\Delta w_i^t$ are the change in weights of the $i^{th}$ worker's local model at $t^{th}$ round.

There are several variations on the simple averaging where the sample size each local model gets is also taken into account \cite{zhou2021communication, mcmahan2017communication}. In sample size weighted averaging (SSFedAvg), the weights of the local networks are multiplied by the proportion of the total dataset they contain. So, if an $i^{th}$ worker contains $d_i$ samples, then the weight aggregation rule is given as,
\begin{equation}
w_{param}^t=\dfrac{\sum\limits_{i=1}^N d_i w_i^t}{N \sum\limits_{i'=1}^N d_{i'}}
\end{equation}

Some other researchers studied the effects of non-i.i.d data settings across the worker nodes and found that the standard variations of FedAvg fail to converge faster. Moreover, the weights of the local models drift apart and the aggregation then requires some adaptive form which prevents the drift. In this context an adaptive aggregation strategy (FedAdp) was proposed \cite{wu2021fast} which takes into account the direction of individual local gradient with respect to the global server gradient. The weight update at each $i^{th}$ worker node at communication round $t$ is given by,
\begin{equation}
w_{i}^{t+1}=w_{param}^{t}-\eta \Delta F_i(w_{param}^{t}),
\end{equation}
where, $F_i()$ is the error function at $i^{th}$ worker node and since the local weights are updated by the global aggregated weight at each communication round, the local weight is updated on top of the parameter server weights calculated at the previous round $w_{param}^{t-1}$. $\Delta F_i(w_{param}^{t})$ denotes the local gradient of $i^{th}$ worker node's model at communication round $t$.
So, the global gradient is given by,
\begin{equation}
\Delta F(w_{param}^{t})=\dfrac{\sum\limits_{i=1}^N d_i \Delta F_i(w_{param}^{t})}{N \sum\limits_{i'=1}^N d_{i'}} .
\end{equation}
The angle of deviation between the global gradient and the local gradient at $i^{th}$ node is given by,
\begin{equation}
\theta_i^t= arc~cos\dfrac{\langle \Delta F(w_{param}^{t}).\Delta F_i(w_{param}^{t})\rangle}{||\Delta F(w_{param}^{t})||||\Delta F_i(w_{param}^{t})||}.
\end{equation}
So, if $\theta_i^t$ is small, it means that the local model is updated in the same direction as the global model and would contribute more to the aggregation. The angles are smoothed and a non-linear decreasing function is imposed on it.
Smoothed angle in radian is given by,
\begin{eqnarray}
\tilde{\theta}_i^t &=& \theta_i^t, ~~~if~t=1\\
 &=& \dfrac{t-1}{t}\tilde{\theta}_i^{t-1}+\dfrac{1}{t}\theta_i^t, ~~~if~t>1,
\end{eqnarray}
and the non-linear decreasing function is given by,
\begin{equation}
f(\tilde{\theta}_i^t)=\alpha(1-e^{-e^{-\alpha(\tilde{\theta}_i^t-1)}}),
\end{equation}
where $\alpha$ is a constant.
The adaptive weight for each $i^{th}$ model in aggregation becomes,
\begin{equation}
\phi_i^t=\dfrac{\sum\limits_{i=1}^N d_i e^{f(\tilde{\theta}_i^t)}}{N \sum\limits_{i'=1}^N d_{i'} f(\tilde{\theta}_{i'}^t)}.
\end{equation}
So, now the weight aggregation rule becomes,
\begin{equation}
w_{param}^t=\sum\limits_{i=1}^N \phi_i w_i^t
\end{equation}
However, even though the method of FedAdp was capable of handling non-i.i.d data setting, it did not take into account cases of extreme imbalance. Moreover, for non-i.i.d data, the number of rounds needed to achieve a desired performance for 10 participating worker nodes were in the order of hundreds.

Another method of aggregation was introduced \cite{zhang2015deep} which could improve the performance of the worker models by allowing them to explore and fluctuate from the parameter's center variables. They argued that since it is a collaborative learning strategy, there might be numerous local optima. The method is called Elastic Averaging (EASGD). In this method, the aggregated weight is a moving average of the local weights.
\begin{equation}\label{eq:EASGD}
w_{param}^t=\gamma w_{param}^{t-1}+ (1-\gamma) \sum\limits_{i=1}^N \dfrac{w_i^t}{N}
\end{equation}
However, in a situation of local imbalance across the workers for non-i.i.d data setting, it is observed that EASGD fails as the number of worker nodes increase.

A detailed study \cite{hyunsoo2021study} on some of the methods for selecting some well-learnt workers to update the global model has been done in the context of imbalanced class distribution. However, such methods aim at discarding a number of workers which are supposedly not good-enough to contribute to the global model. In doing so, they are rejecting the data that the rejected workers hold and the global model does not get to learn from that particular portion of the data.

Researchers have focussed on the problem with data imbalance across worker nodes and have proposed methods to overcome the challenges \cite{duan2020self, duan2019astraea}. A self-balancing methodology known as Astraea is adapted such that the local datasets in each node is rebalanced by augmentation. However, such methods often work on augmenting the local datasets on workers with the minority class. Thus, have an inherent assumption that the local models have seen atleast some actual samples of the minority class. In cases when one or more local models are completely deprived of one or more classes, the problem cannot be handled by minority class data augmentation. Resampling is a popular method to alleviate the imbalance problem by undersampling or over-sampling. However, over-sampling usually is susceptible to errors owing to excess or additional noise, while undersampling lowers the volume of training data from which the model can learn. This makes the above two approaches useless.

Focal loss \cite{lin2017focal} is a specific loss function that has been presented for a binary classification job that imbalances and reshapes a conventional model cross-entropy loss to weigh down well-classified instances and may concentrate on the learning of the hard imbalanced outliers. The focal loss for binary classification is given as,
\begin{equation}\label{eq:OriginalFocalLoss}
E=-\rho \{p_1 (1-p_2)^\xi log (p_2)+p_2 (1-p_1)^\xi log(p_1)\},
\end{equation}
where, $p_1$ and $p_2$ represent the predicted probabilities of the two classes. Usually for a binary classification problem $p_2=1-p_1$. The constant $\rho$ (0 $< \rho\leq$ 1) is a scaling factor which treats the errors made by the two classes differently and $\xi$ determines how much the easily classified samples (majority class in case of imbalanced data) will contribute to the error. For larger values of $\xi$, lesser number of easily classified samples will contribute to error. In a multi-class setting, the focal loss may be written as,
\begin{equation}
E=-\sum\limits_{i=1}^{C}\rho \{p_i (1-\hat{p}_i)^\xi log (\hat{p}_i)\},
\end{equation}
where, $C$ represents the number of classes and $p_i$ and $\hat{p}_i$ represent the predicted probabilities of the class $i$ and not-class $i$. Focal loss has been directly applied to the federated learning paradigm \cite{sarkar2020fed}. However, the method again focused more on selecting well learnt worker nodes rather than handling the imbalance in a direct fashion. A similar loss function called GHMC loss \cite{li2019gradient} (Gradient Harmonizing Mechanism Classification loss) is also proposed to handle outliers in a non-distributed setting. These methods can be straight-forwardly applied to federated learning setting as they consider imbalance of the current worker model. MFE loss \cite{wang2016training} is another cost-sensitive method that requires knowledge about the minority classes. It generates a modified form of loss using the false positives and false negatives. However, such a prior knowledge is impractical in FL. A new loss function called the Ratio Loss \cite{wang2021addressing} has been proposed recently. The method incorporates a monitoring scheme at each communication round and keeps track of the local imbalance in the training data. However, this loss function again needs to identify the minority class in the global dataset as an added burden. Moreover, it is stated by the authors that as the global imbalance increases, the effective performance of the ratio loss degrades. If the global dataset itself poses an outlier detection problem (where the imbalance is severe i.e. the minority class is only 10\% of the dataset), it will fail. Thus, a new loss function is proposed which would modify the local weights in such a manner that even critical imbalance is handled.

Therefore, there is a dire need of an efficient method to handle extreme class imbalance in local nodes. Moreover, the exiting methods on handling class imbalance mostly use pre-trained models. Hence, it has never been studied how the pre-initialisation strategy would affect the learning process in such imbalanced data setting. Because the training data is consolidated in a centralised setting in typical training situations, there is no differentiation between local and global imbalance values. Thus, minimising the negative impact of imbalance is considerably easier than in the cases of decentralised data. It should be noted that in federated learning, each local model training might be considered as a regular centralised training. Intuitively, we might use the existing approaches to handle the local imbalance problem at each round. However, local models exist only momentarily on the local worker nodes and are replaced by the most recent global model after each training cycle. As a result, fixing local imbalances may have little influence. Furthermore, because of the misalignment between local and global imbalances, just implementing known procedures at local devices is often ineffective and may even have a detrimental influence on the global model. Thus, there needs to be a mechanism which can handle each of these local imbalances and the training would work on resolving the global imbalance. The proposed method aims to handle the afore-mentioned issues and also provides a boost in performance. The following section describes the proposed method in detail.

\section{Proposed Method}\label{PropDDL}
The manuscript studies situations where imbalanced data portions may degrade the individual classification ability of the worker nodes. A deep neural network is inherently unsuitable for imbalanced classification due to masking and swamping effects. Thus, researchers have proposed many methods to train a neural network using cost-sensitive backpropagation \cite{lin2017focal, li2019gradient, wang2016training}. As the data is distributed across worker nodes, it becomes relatively difficult to handle the class imbalance. This work does not use any freely available pre-trained models. The local models are first pre-trained using an autoencoder and are then re-trained with class labels. So, each of the local models train over the data in two phases. Autoencoder pre-training is useful for class imbalance situation because the network gets biased towards learning the majority class' features efficiently. So, any data that doesn't fit into the description of the majority class may be picked up as an outlier easily \cite{wang2016auto, hawkins2013identification}. The re-training is just done to learn the classification between the majority and minority classes. The following section describes the proposed method to deal with severe class imbalance in distributed setting.

\subsection{Division of Data Across Worker Nodes}\label{ddldp}

The data is arbitrarily distributed and resembles a practical scenario where a user has no control over the number of samples each node gets. In both the cases, the datasets were partitioned and given to each worker node. Partitioning the data into non-overlapping portions makes the data across local workers `non-independent'. To remove the `identically-distributed' assumption of i.i.d, the global data clustered and each cluster represented a partition of the dataset. For this, we used $k$-means clustering where $k$ is equal to the number of worker nodes. Thus, it was ensured that the local datasets that the worker nodes get are non-i.i.d. As the number of clusters grow (which is same as increasing the number of local workers), the probability of each cluster having samples from only a single class increases \cite{rosenberg2007v}. Thus, as the number of local workers grow, there is an increased chance that the local datasets are mostly of a single class only. This method of data division is merely mimicking a practical data distribution scenario across the workers which assumes that the data generated at each local worker is mostly similar.

To understand the practical implication of this assumption, let us consider an example where there are two branches of a company located across two different locations in the world. It may happen that for one of the locations, the workers mostly relatively younger and for the other location the workers are mostly older. Thus, the first company can produce goods at a much faster rate with less errors while the second company produces goods at a slower pace and more errors than the first one. Thus, the data obtained from the first location will have more samples and majority of the samples will pass the quality check. On the other hand, the data from other location will have lesser samples and majority of the samples will not pass the quality check. So, if one attempts to learn the model of quality check from both the companies, he or she needs to learn in an imbalanced and distributed setting simultaneously.
\subsection{Cost-sensitive Neural Network Training}
The focal loss may be extended to a cost-sensitive federated learning scenario by incorporating an adaptive scaling factor in each individual local model. Since each model gets different number of samples and the class distribution in each model is different, we need to adapt the value to $\rho$ accordingly. So, an adaptive focal loss function is proposed in this context. Suppose there are  For each local model $j$, the proposed focal loss is,
\begin{equation}\label{eq:ProposedFocalLoss}
E_j=-\sum\limits_{i=1}^{C}\rho_j \{p_i (1-\hat{p}_i)^\xi log (\hat{p}_i)\},
\end{equation}
where $\rho_j$ is the adaptive scaling factor of the $j^{th}$, local model. Since, the main motive for imbalanced data is to avoid the masking phenomena (labelling all the samples with the majority class label), it is intuitive to express $\rho_j$ in terms of an crude class imbalance ratio $m_j$ for each local dataset, such that,
\begin{equation}
\rho_j=\dfrac{a}{(1+e^{-b*(m_j-1)})}
\end{equation}
where,
\begin{equation}
m_j=1-\dfrac{d_j^{majority}}{d_j^{total}}.
\end{equation}
Here, $d_j^{majority}$ and $d_j^{total}$ represent the number of samples in the majority class of the local dataset in $j^{th}$ worker and the total number of samples in the local dataset in $j^{th}$ worker respectively. A crude class imbalance is much simpler than having class wise imbalance ratio for extremely imbalanced datasets. This is because it is only the majority class which overpowers the minority classes. The other minority classes do not affect each other much as they are themselves smaller in number. For the present manuscript the values of $a$ and $b$ were chosen to be 2 and 3 respectively. So, if a local dataset has all the samples from the majority class only ($m_j=0$), the local model would be completely masked and must not be able to make decent contributions to the error. So, if $m_j=0$ then instead of pushing $\rho_j$ to zero, we are just mapping it to a value close to zero (but not exactly zero) using the scaled sigmoid function. Pushing $\rho_j$ to zero would mean that the model is absolutely correct and will not add anything to the server weights. This would imply a complete loss of information from one worker which is undesirable. Thus, the value of $\rho_j$ being close to zero would mean that the local models would atleast have some small contribution to the server model.

Another intuitive explanation for this adaptive function is that each model is independently moving towards convergence. This rate of movement depends on the individual loss they incur. Now, if a local model is overwhelmed with severe imbalance, then the constant $\rho_j$ would prevent that model to overfit the majority class too much. This would imply that more severe the imbalance a local model would face, the lesser it would deviate from the last global update, but the information would not be completely lost. This would help the global weights to not drift randomly. It can be proved that there would always atleast be one local data portion that will not have the severity in imbalance, and the model trained on that particular data would always govern the global weights. However, that model alone is not enough to learn the entire data space. So, the other models which are inflicted with severe imbalance would add the fine tuning to the global weights accordingly.
\begin{theorem}
If the global dataset $P$ has an outlier ratio $r$, and it is partitioned across $N$ divisions namely, $P_1$,  $P_2$...  $P_N$, then there would be atleast one division $P_i$ which would have outlier ratio $\geq r$.
\end{theorem}
\begin{proof}
Suppose the global data having $n_1$ samples from the minority or outlier class (denoted as class 1) and $n_2$ samples from the majority or inlier class (denoted as class 2), is partitioned into $N$ local datasets. The global outlier ratio $r$ is given by,
\begin{equation}\label{eq:00}
r=\dfrac{n_1}{n_1+n_2}
\end{equation}
For the $i^{th}$ partition (local dataset) $P_i$, let there are $n_{i,1}$ samples from the minority class and $n_{i,2}$ samples from the majority class. Then the outlier ratio of that partition is given by,
\begin{equation}\label{eq:01}
r_i=\dfrac{n_{i,1}}{n_{i,1}+n_{i,2}}
\end{equation}
 Now, as the global dataset is partitioned into $N$ local data partitions,
 \begin{equation}\label{eq:1}
 \sum\limits_{i=1}^N n_{i,1}= n_1,~and
 \end{equation}
 \begin{equation}\label{eq:2}
 \sum\limits_{i=1}^N n_{i,2}= n_2.
 \end{equation}
 Equation \ref{eq:01} can be written as,
 \begin{equation}
 n_{i,2}=(1-r_i)n_{i,1},
 \end{equation}
 and, from equation \ref{eq:00}, we get
 \begin{equation}\label{eq:02}
n_2=(1-r)n_1.
\end{equation}
 Replacing equations\ref{eq:1} and \ref{eq:2} in equation \ref{eq:02}, we get,
 \begin{equation}\label{eq:contra1}
 r=\dfrac{\sum\limits_{i=1}^N r_i n_{i,1}}{\sum\limits_{i=1}^N n_{i,1}}
 \end{equation}
For argument, if we assume that for all partitions, the outlier ratio $r_i$ is strictly less than the the global outlier ratio $r$, then,
 \begin{equation}\label{eq:ineq}
 r_i<r,~for~i =~1,~2,~3,~...,~N.
 \end{equation}
Multiplying both sides of equation \ref{eq:ineq} by $m_{i,1}$, summing over all the partitions on both sides and rearranging we get,
 \begin{equation}\label{eq:contra2}
 r>\dfrac{\sum\limits_{i=1}^N r_i n_{i,1}}{\sum\limits_{i=1}^N n_{i,1}}
 \end{equation}
It can be seen that the equations \ref{eq:contra1} and \ref{eq:contra2} contradict each other. Thus, there exist atleast one partition of local data which has an outlier ratio $\geq r$.
\end{proof}
\subsection{Overall Method}
EASGD has been proven to be stable \cite{zhang2015deep} in cases of distributed deep learning. It induces an elastic force that does not allow the local model's weight matrix to deviate too much from the global weight matrix whereas also allowing some independence. The proposed method incorporates the concept of EASGD with the proposed adaptive focal loss [equation \ref{eq:ProposedFocalLoss}]. The training of the network is done in two phases.
\begin{enumerate}[(i)]
\item pre-training the network using deep autoencoders and
\item retraining the entire network using class labels.
\end{enumerate}

While pre-training the network, since there are no class labels involved, and the models merely learn to reconstruct the input, the unnecessary involvement of the adaptive cost-sensitive loss function is not needed. So, the autoencoders are trained by the standard EASGD (equation \ref{eq:EASGD}) to minimise the reconstructional mean squared error loss. Now, for the retraining, the model must learn the discrimination between classes even in presence of severe class imbalance. So, retraining needs the cost-sensitive loss function along with EASGD. The cost-sensitive adaptive focal loss function determines how much independence the individual local model would get in determining the optimum weights depending on the class imbalance they witness.

The pre-training without cost-sensitivity allows the autoencoders to distinguish between the majority and the minority classes. It has been proven that autoencoders map similar values to similar ranges \cite{wang2016auto} and thus for classification in such severe imbalanced datasets, they are very useful. Autoencoders are extensively employed in outlier identification \cite{lin2019probabilistic, sun2019stacked}, where an input with a substantial reconstruction error is identified as out-of-distribution, as the quality of reconstruction is expected to decrease for inputs that are considerably different from majority of the training data. The traditional AE-based pre-training is also favoured since it is highly stable and simple, as well as providing a decent initialisation to begin with. So, the retraining is to just fine tune the local models according to the proportion of the different classes present in the local dataset.
\section{Experiments and Results}
In order to establish the supremacy of the proposed model, the study explores many different angles on an empirical note. For a fair comparison between the methods, the numbers of models were increased. The model architecture were kept same for all the methods. All the methods were compared with the same partitions of the data produced by kmeans clustering (as assumed in section \ref{ddldp}). Each of the methods were run for only 20 communication rounds between the parameter server and the workers. The datasets used were highly imbalanced in nature. We have chosen eight binary-class outlier detection datasets for this purpose which had the outlier class samples representing less than 10\% of the total dataset. The description of the datasets are provided in table \ref{Tab:DDLdatasetdescr}].
\begin{table}[tbp]
\begin{center}
\begin{tabular}{llll}
\hline\hline
\textbf{Name of}       & \textbf{Number of} & \textbf{Number of} & \textbf{Percentage of} \\\hline\hline
\textbf{the   dataset} & \textbf{Samples}   & \textbf{features}  & \textbf{Outliers}      \\
Credit   Card          & 284807             & 29                 & 0.1761                 \\
HTRU2                  & 17898              & 8                  & 9.11                   \\
MNIST                  & 7603               & 100                & 8.7373                 \\
Pulsar                 & 17898              & 8                  & 9.204119               \\
Forest   Cover         & 286048             & 10                 & 0.9484                 \\
Mammogram              & 7828               & 6                  & 2.325373               \\
Pen Digits             & 4808               & 16                 & 2.271404               \\
ALOI                   & 34999              & 28                 & 3.016121               \\\hline\hline
\end{tabular}
\end{center}
\caption{Dataset description}
\label{Tab:DDLdatasetdescr}
\end{table}

The following sections provide the results of the experiments conducted. The results reported are the median over ten independent runs.
\subsection{How Pre-initialisation affects Outlier Detection?}

The initialization phase is essential to the model's eventual output and mandates an effective approach. Traditionally, weight initialization entailed employing smaller random values (close to zero). However, over the last decade, better explicit heuristics that utilise network architecture information, for example, the type of activation being utilised or the number of inputs to the network, have been developed. These rather specific strategies can lead to more successful neural network model construction when using the stochastic gradient descent optimization algorithm. This optimization procedure necessitates a starting point in the universe of possible weight values wherein the optimization may commence. Deep model training is a challenging enough process that the initialization method used has a significant impact on most techniques. The starting point can decide whether or not the process converges at all, with certain initial values being so unstable that the algorithm runs into numerical challenges and fails entirely. Many initialisation strategies are gradually developed. Some popular methods used these days are Random Normal \cite{drago1992statistically}, Random Uniform \cite{drago1992statistically}, Xavier-Glorot Normal \cite{glorot2011deep}, Xavier-Glorot Uniform \cite{glorot2011deep}, He Normal \cite{he2015delving} and He Uniform \cite{he2015delving}. AEs were also proposed as a pre-initialisation strategy to avoid the problem of vanishing and exploding gradients \cite{goodfellow2016deep}. While techniques like AE pre-training certainly laid the groundwork for many essential principles in today's deep learning methodologies, the compelling necessity for pre-training neural networks has waned in recent times. This was largely due to various advancements in regularisation, network architectures, and enhanced optimization methods. Despite these advancements, training deep neural networks to generalise well to a wide range of complicated tasks remains a major issue. One of these is the issue of imbalanced datasets. Neural networks get biased towards the majority class. However, AEs are well suited for outlier detection tasks. Hence, in order to strengthen the position of AE pre-training in outlier detection tasks, it was compared with the popular initialisation strategies. The models use various initialisation schemes and the training for classification is done using the focal loss function [equation \ref{eq:OriginalFocalLoss}] (because for a single model, the equations \ref{eq:OriginalFocalLoss} and \ref{eq:ProposedFocalLoss} are identical). Table \ref{Tab:InitiaIm} gives a comparison of how the various initialisations perform for severely imbalanced datasets (where the minority class has less than 10\% representation) with respect to all the four metrics namely ROC-AUC (Area Under the Curve of Receiver Operating Characteristic) \cite{hawkins2013identification}, F-Score  \cite{nicolas2017scala}, DR (Detection Rate) \cite{chakraborty2019integration} and G-Mean \cite{nicolas2017scala}. It can be seen that AE pre-training added an advantage to the performance of a deep neural network for all the eight datasets under consideration.

\begin{table}[tbp]
\begin{center}
\resizebox{\textwidth}{!}
{
\begin{tabular}{cllllllll}
\hline\hline
\multirow{3}{*}{\textbf{Dataset}} & \multicolumn{1}{c}{\multirow{3}{*}{\textbf{Metrics}}} & \textbf{}            & \textbf{}       & \textbf{}        & \textbf{Xavier-} & \textbf{Xavier-} & \textbf{}       & \textbf{}        \\
                                  & \multicolumn{1}{c}{}                                  & \textbf{AE-}         & \textbf{Random} & \textbf{Random}  & \textbf{Glorot}  & \textbf{Glorot}  & \textbf{He}     & \textbf{He}      \\
                                  & \multicolumn{1}{c}{}                                  & \textbf{Pretraining} & \textbf{Normal} & \textbf{Uniform} & \textbf{Normal}  & \textbf{Uniform} & \textbf{Normal} & \textbf{Uniform} \\\hline\hline
\multirow{4}{*}{Credit Card}      & \textbf{ROC}                                          & \textbf{0.9537}      & 0.5000          & 0.9339           & 0.9365           & 0.9379           & 0.9323          & 0.9461           \\
                                  & \textbf{F-Score}                                      & \textbf{0.9346}      & 0.5000          & 0.9036           & 0.9159           & 0.9176           & 0.9098          & 0.9104           \\
                                  & \textbf{DR}                                           & \textbf{0.8083}      & 0.0000          & 0.7870           & 0.7799           & 0.7799           & 0.7870          & 0.7941           \\
                                  & \textbf{G-Mean}                                       & \textbf{0.9351}      & 0.5000          & 0.9037           & 0.9163           & 0.9180           & 0.9100          & 0.9105           \\\hline
\multirow{4}{*}{HTRU2}            & \textbf{ROC}                                          & \textbf{0.9756}      & 0.5000          & 0.9726           & 0.9737           & 0.9736           & 0.9737          & 0.9731           \\
                                  & \textbf{F-Score}                                      & \textbf{0.9425}      & 0.0000          & 0.9309           & 0.9340           & 0.9348           & 0.9336          & 0.9316           \\
                                  & \textbf{DR}                                           & \textbf{0.8491}      & 0.0000          & 0.8450           & 0.8343           & 0.8407           & 0.8364          & 0.8257           \\
                                  & \textbf{G-Mean}                                       & \textbf{0.9426}      & 0.0000          & 0.9309           & 0.9342           & 0.9349           & 0.9337          & 0.9317           \\\hline
\multirow{4}{*}{MNIST}            & \textbf{ROC}                                          & \textbf{0.9952}      & 0.9831          & 0.8459           & 0.9907           & 0.9885           & 0.9856          & 0.9885           \\
                                  & \textbf{F-Score}                                      & \textbf{0.9548}      & 0.9359          & 0.7506           & 0.9340           & 0.9422           & 0.9278          & 0.9355           \\
                                  & \textbf{DR}                                           & \textbf{0.9339}      & 0.8734          & 0.3696           & 0.8550           & 0.8920           & 0.8595          & 0.8642           \\
                                  & \textbf{G-Mean}                                       & \textbf{0.9549}      & 0.9359          & 0.7537           & 0.9340           & 0.9422           & 0.9278          & 0.9355           \\\hline
\multirow{4}{*}{Pulsar}           & \textbf{ROC}                                          & \textbf{0.9797}      & 0.5000          & 0.9752           & 0.9733           & 0.9693           & 0.9723          & 0.9758           \\
                                  & \textbf{F-Score}                                      & \textbf{0.9528}      & 0.0000          & 0.9400           & 0.9449           & 0.9250           & 0.9325          & 0.9387           \\
                                  & \textbf{DR}                                           & \textbf{0.8590}      & 0.0000          & 0.8153           & 0.8369           & 0.7669           & 0.7909          & 0.8150           \\
                                  & \textbf{G-Mean}                                       & \textbf{0.9530}      & 0.0000          & 0.9405           & 0.9452           & 0.9257           & 0.9330          & 0.9391           \\\hline
\multirow{4}{*}{Mammogram}        & \textbf{ROC}                                          & \textbf{0.9389}      & 0.5000          & 0.9223           & 0.9306           & 0.9311           & 0.9341          & 0.9231           \\
                                  & \textbf{F-Score}                                      & \textbf{0.8817}      & 0.0000          & 0.8565           & 0.8362           & 0.8365           & 0.8413          & 0.7470           \\
                                  & \textbf{DR}                                           & \textbf{0.6980}      & 0.0000          & 0.6844           & 0.5892           & 0.6163           & 0.5756          & 0.3731           \\
                                  & \textbf{G-Mean}                                       & \textbf{0.8823}      & 0.0000          & 0.8566           & 0.8373           & 0.8370           & 0.8431          & 0.7499           \\\hline
\multirow{4}{*}{Pendigits}        & \textbf{ROC}                                          & \textbf{0.9876}      & 0.5000          & 0.5000           & 0.9706           & 0.9779           & 0.9753          & 0.9761           \\
                                  & \textbf{F-Score}                                      & \textbf{0.9660}      & 0.0000          & 0.0000           & 0.9339           & 0.9486           & 0.9601          & 0.9495           \\
                                  & \textbf{DR}                                           & \textbf{0.9763}      & 0.4888          & 0.4888           & 0.9049           & 0.9636           & 0.9755          & 0.9543           \\
                                  & \textbf{G-Mean}                                       & \textbf{0.9112}      & 0.0000          & 0.0000           & 0.9334           & 0.8670           & 0.8891          & 0.8891           \\\hline
\multirow{4}{*}{ALOI}             & \textbf{ROC}                                          & \textbf{1.0000}      & 0.5000          & \textbf{1.0000}  & \textbf{1.0000}  & \textbf{1.0000}  & \textbf{1.0000} & \textbf{1.0000}  \\
                                  & \textbf{F-Score}                                      & \textbf{0.9994}      & 0.0000          & 0.9966           & 0.9983           & 0.9983           & 0.9977          & 0.9977           \\
                                  & \textbf{DR}                                           & \textbf{1.0000}      & 0.0000          & 0.9908           & 0.9954           & 0.9954           & 0.9931          & 0.9931           \\
                                  & \textbf{G-Mean}                                       & \textbf{0.9994}      & 0.0000          & 0.9966           & 0.9983           & 0.9983           & 0.9977          & 0.9977   \\\hline\hline
\end{tabular}}
\caption{Effect of Various Network Initialisations on the Imbalanced Classification Performance}
\label{Tab:InitiaIm}
\end{center}
\end{table}

It can be seen that AE pretrained models are better in performance for outlier detection tasks. This supports the findings \cite{chakraborty2019integration} that AE are better feature extractors for outlier detection tasks. Thus, the experiments can now be extended to multiple models in a federated setting as described in the next sections.
\subsection{Data Division and Critical Imbalance}
Data division in a non-i.i.d setting violates two main properties of i.i.d \cite{jacobs1992independent}. Firstly, the data is divided in an insufficiently random order, which generates the violation on independence. Secondly, there is varying amounts of data in each division (with respect to class labels or sample quantity, or both). This clearly implies that each data division must follow a separate distribution which is not identical to the global data distribution. Let the global data distribution is $\mathcal{D}$, then in an i.i.d data division, every $i^{th}$ portion of data $\mathcal{P}^{(i)}$ has the same distribution $\mathcal{D}$. However, for non-i.i.d data division, every $i^{th}$ portion of data $\mathcal{P}^{(i)}$ might follow a different distribution $\mathcal{D}^{(i)}$, which means that there is an underlying structure in each data portion. To mimic this property, one may safely consider clustering the data into $k$ groups. In this work, the clustering has been done using $k$-means clustering algorithm. This preserves the non-i.i.d properties that the data portions do not follow an identical distribution (as each cluster consists of similar samples and samples from different clusters are dissimilar to each other) and are also not independent (as both the proportion of class labels and number of samples vary in each cluster). The number $k$ is not always equal to the number of classes in the dataset. This is because any clustering algorithm groups data according to similarity of the samples and not according to the labels. Homogeneity metric \cite{rosenberg2007v} indicates the conditional entropy of a clustering algorithm in identifying the correct classification and is bounded between 0 to 1. This means that if more clusters contain samples from only one class, the score will increase. This is evident from our findings too [Figure \ref{Homogeneity}].
\begin{figure}
  \centering
      \includegraphics[width=1\textwidth]{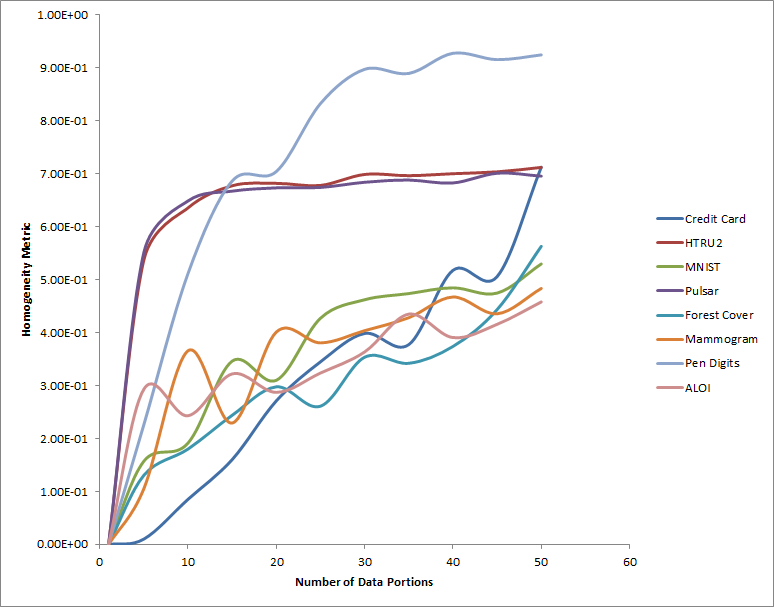}
  \caption{Variation of the Homogeneity Metric with the Number of Data Portions}
  \label{Homogeneity}
\end{figure}
It can be seen that as the number of data portions increased (or number of worker nodes increased), there were multiple clusters which had no samples from the minority class [Figure \ref{MinorityPerc}] (result shown for Credit Card dataset).
\begin{figure}
  \centering
      \includegraphics[width=1\textwidth]{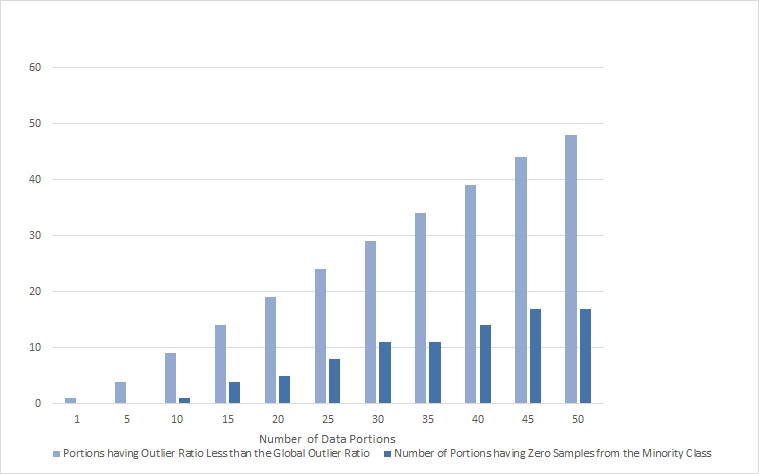}
  \caption{Imbalance Becomes Critical with the Increasing Number of Data Portions}
  \label{MinorityPerc}
\end{figure}
The number of data portions having imbalance ratio below the global imbalance ratio kept on increasing with the number of data portions (or number of worker nodes) [Figure \ref{MinorityPerc}]. This clearly indicates a critical imbalance situation that federated learning suffers from.
\subsection{Performance under Critical Imbalance}
As can be seen from Figure \ref{MinorityPerc}, with increasing number of partitions, there are multiple local datasets which suffered critical imbalance (i.e. have no samples from the outlier class). Under such conditions, we need to analyse how the performance of the global model varies with each communication round. For experimental purposes, we have taken 10 worker nodes for Credit Card dataset. It can be seen from Figure \ref{MinorityPerc} that for 10 worker nodes, there was 1 data portions which had zero samples from the outlier class and 9 data portions having outlier ratio less than the global outlier ratio (including the one which has no outlier samples). Figure \ref{Stability} shows the F-Score obtained on the test data for each local model before each communication round and global model after each communication round. This is clear indication of the stability of the elastic averaging \cite{zhang2015deep}.
\begin{figure}
  \centering
      \includegraphics[width=0.7\textwidth]{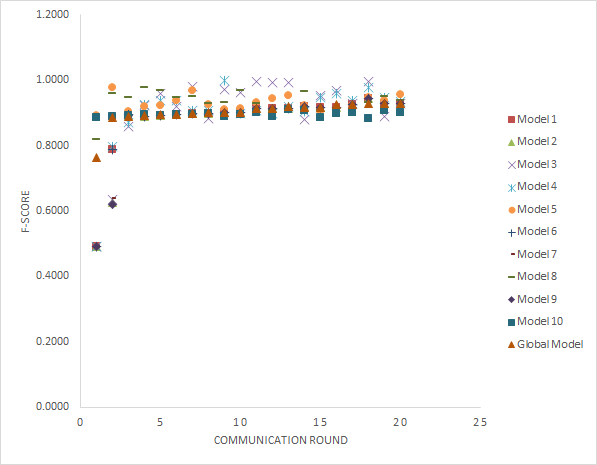}
  \caption{F-Score of Individual Local Models and the Global Model vs Communication Round}
  \label{Stability}
\end{figure}

It can be seen from figure \ref{Stability} that the individual models gradually coincided with similar performance on the test data with each communication round. It can be seen that several models were not performing well in the initial rounds. Models 1, 2, 3, 4, 6, 7 and 9 were only capable of detecting the samples from the majority class after the first round. However, due to the learning obtained from the other well performing models, they gradually shifted to a better performance in classifying the samples. This clearly indicates that even though there are some under-performing models in the learner pool, the global model will not be so severely affected by it and there will be a gradual learning eventually.
\subsection{Number of workers vs overall performance}
Under the above conditions, another important aspect is to study how the increasing number of worker nodes affect the overall performance of the global model. In the case presented, the total data is always constant and each worker gets a unique partition of the dataset. Thus, as the number of worker nodes increase, the number of samples in each worker node will decrease. So, increasing the number of workers above a certain limit is not possible so as to prevent the unlikely condition of models getting insufficient samples for training. It has been seen that all the existing methods tend to be biased towards the majority class as the number of workers grow. It is noteworthy that increasing the number of workers in the present assumption of experiments that the data samples are partitioned (i.e. increasing the number of workers would decrease the number of samples in each partition), not having an enough quantity of the data in each worker node will be an inevitable cause of performance degradation. Hence, the study mainly shows that the proposed method yields a consistent performance with increasing worker nodes while the existing methods hit a saturation point after a certain limit. The proposed method has been compared with the existing methods used for outlier detection in federated learning: EASGD (with cross-entropy loss) \cite{zhang2015deep}, Astraea\cite{duan2019astraea}, FedFocal \cite{sarkar2020fed}, and Ratio Loss \cite{wang2021addressing}.
\begin{figure}
  \centering
      \includegraphics[width=0.7\textwidth]{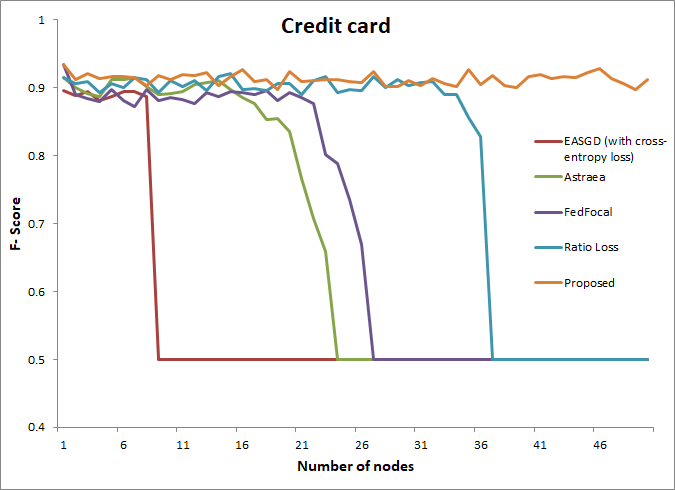}
  \caption{Variation of F-Score Obtained by Various Methods for Different Numbers of Worker Nodes on Credit Card Dataset.}
  \label{Result1}
\end{figure}
\begin{figure}
  \centering
      \includegraphics[width=0.7\textwidth]{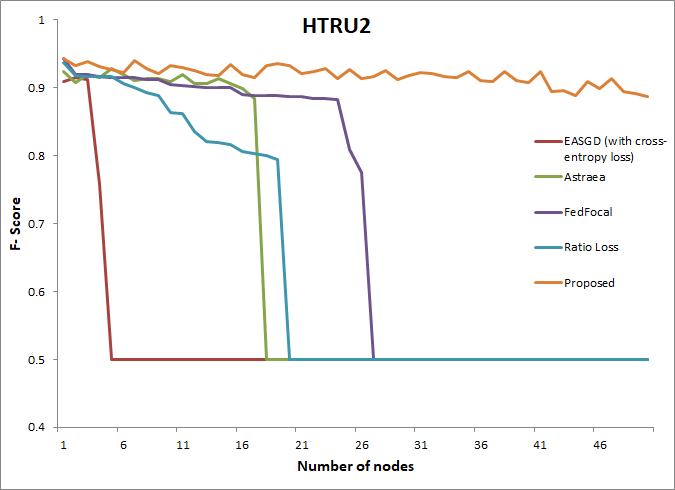}
  \caption{Variation of F-Score Obtained by Various Methods for Different Numbers of Worker Nodes on HTRU2 Dataset.}
  \label{Result2}
\end{figure}
\begin{figure}
  \centering
      \includegraphics[width=0.7\textwidth]{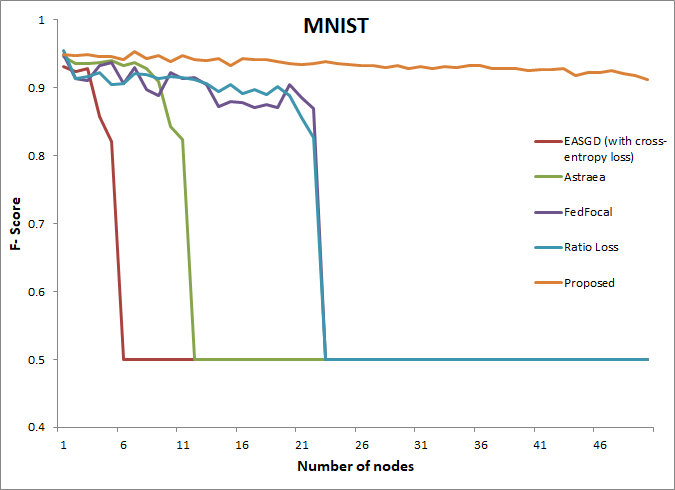}
  \caption{Variation of F-Score Obtained by Various Methods for Different Numbers of Worker Nodes on MNIST Dataset.}
  \label{Result3}
\end{figure}
\begin{figure}
  \centering
      \includegraphics[width=0.7\textwidth]{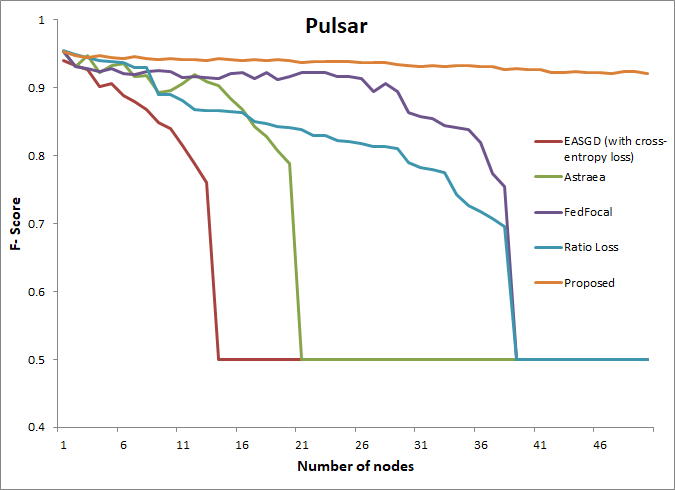}
  \caption{Variation of F-Score Obtained by Various Methods for Different Numbers of Worker Nodes on Pulsar Dataset.}
  \label{Result4}
\end{figure}
\begin{figure}
  \centering
      \includegraphics[width=0.7\textwidth]{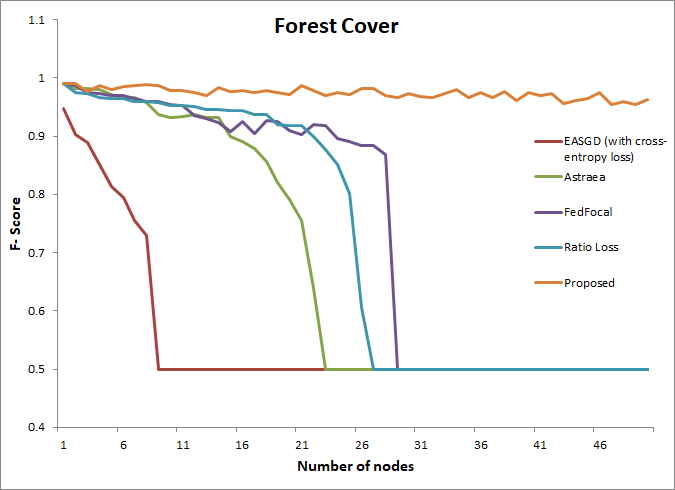}
  \caption{Variation of F-Score Obtained by Various Methods for Different Numbers of Worker Nodes on Forest Cover Dataset.}
  \label{Result5}
\end{figure}
\begin{figure}
  \centering
      \includegraphics[width=0.7\textwidth]{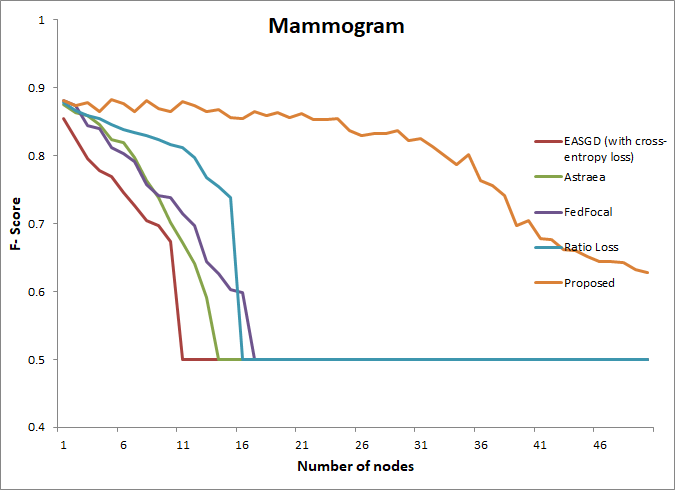}
  \caption{Variation of F-Score Obtained by Various Methods for Different Numbers of Worker Nodes on Mammogram Dataset.}
  \label{Result6}
\end{figure}
\begin{figure}
  \centering
      \includegraphics[width=0.7\textwidth]{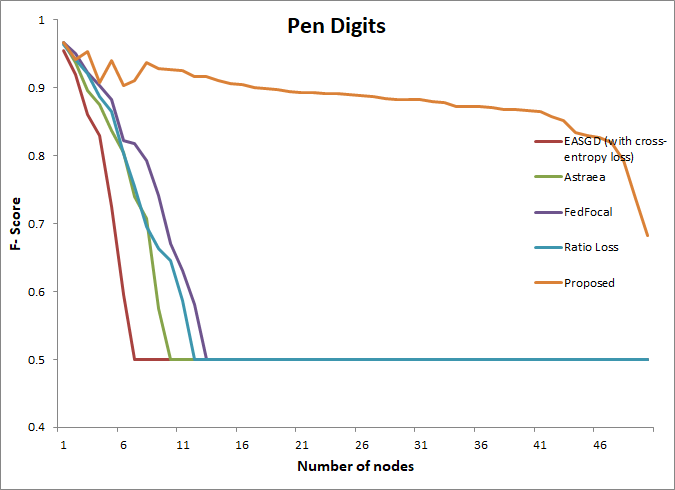}
  \caption{Variation of F-Score Obtained by Various Methods for Different Numbers of Worker Nodes on Pen Digits Dataset.}
  \label{Result7}
\end{figure}
\begin{figure}
  \centering
      \includegraphics[width=0.7\textwidth]{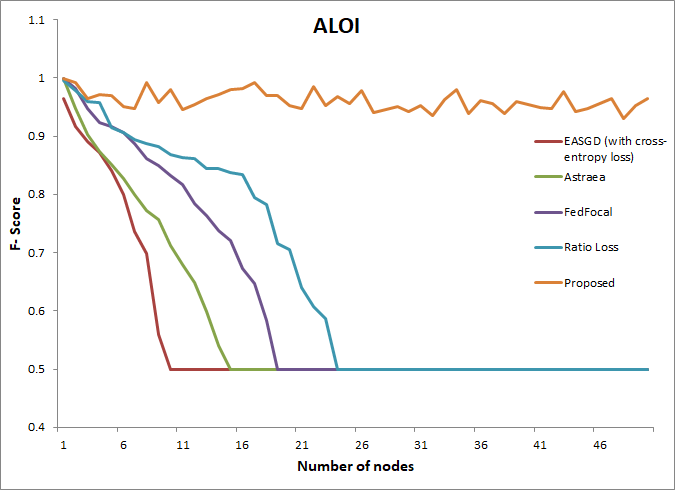}
  \caption{Variation of F-Score Obtained by Various Methods for Different Numbers of Worker Nodes on ALOI Dataset.}
  \label{Result8}
\end{figure}

As observed from the above experiments, the outlier detection abilities of EASGD (with cross-entropy loss) suffers a fast degradation with increasing number of nodes. Astraea improved on the detection abilities but with increasing number of nodes, showed incompetence. The Ratio Loss method and the method FedFocal were more or less similar in performance and were able to detect outliers fairly upto an appreciable number of nodes before suffering failure. The proposed method had the best performance among all the methods and the performance was more or less steady with increasing number of nodes and was still able to detect outliers while all the detection abilities of the other existing methods already collapsed. There was a gradual decrease in the performance of the proposed method for some datasets due to the lack of enough training samples in each local partition with the increasing number of nodes. This is an expected consequence of our assumption that the total number of samples in the global dataset is constant.
\section{Discussion and Conclusion}
In this article it has been shown that AE pre-training is better for outlier detection tasks. The study is then extended to the use of AE pre-training in distributed setting. In this context an adaptive version of focal loss function is introduced to handle the critical imbalance in the distributed data setting. The main assumptions in the experiments were that the data available at each local node is mutually exclusive. If one local node fails, it is assumed that there is a backup node for each local node which would continue with the learning process. Under such assumptions, it has been observed that the proposed method can handle the critical imbalance better than the existing methods. The experiments show that even when some local datasets did not have any samples from the outlier class, the proposed method drove the global model and each of the local models to similar performance on the test data.

However, there might arise an extreme case of imbalance where one local model may contain a lot of outlier samples such that the outlier class becomes the majority for that particular local dataset and all the other local datasets do not contain any outlier sample at all. In that case, the present solution might not work as it relies on the outlier ratio. Such an extreme case would practically imply that all the anomalies are generated at the same location while others do not generate any at all. This problem needs to be handled in a different light and has been set aside as a direction of future research. The research may also be extended to handle multiple classes where the local imbalance is critical and the global imbalance may or may not be present.

\bibliographystyle{plain}
\bibliography{Main}
\end{document}